\documentclass[12pt]{l4dc2023}

\title[Certifying Safety in Reinforcement Learning]{Certifying Safety in Reinforcement Learning under Adversarial Perturbation Attacks}
\usepackage{times}  %
\usepackage{helvet}  %
\usepackage{courier}  %
\usepackage{graphicx} %
\usepackage{natbib}  %
\usepackage{amssymb}
\usepackage{algorithm}
\usepackage{algorithmic}
\usepackage[flushleft]{threeparttable}
\usepackage{xcolor}

\usepackage{newfloat}
\usepackage{listings}

\author{%
 \Name{Junlin Wu} \Email{junlin.wu@wustl.edu}\\
 \Name{Hussein Sibai} \Email{sibai@wustl.edu}\\
 \Name{Yevgeniy Vorobeychik} \Email{yvorobeychik@wustl.edu}\\
 \addr Computer Science \& Engineering\\
    Washington University in Saint Louis%
}

\begin{document}

\maketitle

\begin{abstract}
Function approximation has enabled remarkable advances in applying reinforcement learning (RL) techniques in environments with high-dimensional inputs, such as images, in an end-to-end fashion, mapping such inputs directly to low-level control.
Nevertheless, these have proved vulnerable to small adversarial input perturbations.
A number of approaches for improving or certifying robustness of end-to-end RL to adversarial perturbations have emerged as a result, focusing on cumulative reward.
However, what is often at stake in adversarial scenarios is the violation of fundamental properties, such as safety, rather than the overall reward that combines safety with efficiency.
Moreover, properties such as safety can only be defined with respect to true state, rather than the high-dimensional raw inputs to end-to-end policies.
To disentangle nominal efficiency and adversarial safety, we situate RL in deterministic partially-observable Markov decision processes (POMDPs) with the goal of maximizing cumulative reward subject to safety constraints.
We then propose a \emph{partially-supervised reinforcement learning (PSRL)} framework that takes advantage of an additional assumption that the true state of the POMDP is known at \emph{training} time.
We present the first approach for certifying safety of PSRL policies under adversarial input perturbations, and two adversarial training approaches that make direct use of PSRL.
Our experiments demonstrate both the efficacy of the proposed approach for certifying safety in adversarial environments, and the value of the PSRL framework coupled with adversarial training in improving certified safety while preserving high nominal reward and high-quality predictions of true state.
\end{abstract}


\section{Introduction}

Recent years have seen remarkable advances in reinforcement learning (RL) techniques using deep neural networks to represent either value and $Q$ functions, or policies, with applications ranging from autonomous driving to healthcare~\citep{panesar2019machine,Kiran21}.
As we aspire to transition advances to practice in high-stakes applications, however, safety becomes a critical concern.
Indeed, a series of demonstrations have shown that small perturbations to observations that constitute inputs into the deep neural networks can successfully subvert learned policies to nearly arbitrary adversarial ends~\citep{pattanaik2017robust,sun2020stealthy}.
In response, a number of approaches have been proposed to augment traditional reinforcement learning algorithms in order to improve empirical robustness~\citep{lutjens2020certified,oikarinen2021robust,wu2022robust,zhang2020robust}.
In parallel, several recent efforts proposed techniques for \emph{certifying} robustness of learned policies to adversarial perturbations~\citep{kumar2022policy,wu2022crop}.

What seems notably missing in prior work on robustness of RL to adversarial perturbations is any explicit consideration of safety, except as abstractly embedded in the reward signal along with other important, but conceptually orthogonal, aspects such as efficiency.
Conflating these two issues is consequential: in practice, adversarial events are quite rare, and when they occur, we are most concerned about safety, rather than efficiency.
On the other hand, safety verification is a central concern in formal analysis of dynamical systems and control~\citep{C2E2paper,FanQMV:CAV2017,CAS13,alex02hscc,HuangFMMK14}.
An important feature in the latter is that state is relatively low-dimensional and interpretable, so that safety properties can be specified with respect to it.
In contrast, robust reinforcement learning and robustness certification methods have been designed for end-to-end approaches that map high-dimensional uninterpretable perceptual inputs directly into low-level controls.
However, as safety properties cannot be meaningfully specified over such inputs, it is not explicitly considered.

To bridge the gap between robustness certification approaches in adversarial RL and verification methods in dynamical systems and control, we begin by explicitly modeling the environment as a 
deterministic partially-observable Markov decision process (POMDP).
States in this POMDP are low-dimensional and interpretable (and, therefore, amenable to safety specification); observations are high-dimensional and uniterpretable.
Thus, conventional control would verify properties of the true state of this POMDP (usually, with additional assumptions about its structure), whereas conventional end-to-end RL would map observations directly to actions, effectively treating observations as state.
Like conventional verification, we define safety with respect to the true state, and set the goal of the agent as maximizing cumulative reward over a finite horizon subject to safety constraints.
In addition, although the true state of a POMDP is not observable at decision time, we assume that it can be observed during \emph{training}.
For any safety-critical application, this assumption is natural, as one must carefully log and annotate the training phase just for the purposes of testing and debugging.

Our first contribution is a partially-supervised reinforcement learning framework (PSRL) that couples supervised learning used to predict state from observations with reinforcement learning that learns a control policy over a low-dimensional \emph{predicted} system state. 
Supervision can in principle be done independently from RL (e.g., one can use a pretrained prediction model, e.g., mapping an image to locations of obstacles on the road), but it would typically be best to at least fine-tune such models during RL to better match performance to the policy-induced input distribution.
The PSRL framework can be viewed as a simplified model of the far more complex compositions of perception, learning and control,  that are commonly adopted in safety-critical environments, such as autonomous driving.

Our second contribution is an approach for adversarial safety certification of PSRL policies over a finite time horizon.
Specifically, we develop an algorithm that computes a lower bound on the greatest magnitude of adversarial perturbations to \emph{observations} such that safety constraints (specified with respect to true \emph{states}) are not violated.

Finally, we propose adversarial training approaches that make use of the PSRL framework.
An important tension in this context is between end-to-end adversarial training updates, which is most directly tuned to decisions, and adversarial training updates to the supervised and RL components, which most leverages the available supervised signal.
We investigate several variations of adversarial training that balance these considerations in alternative ways, and find that a hybrid approach that combines end-to-end and supervised updates yields a good balance of high nominal reward, adversarial certified safety, and accurate state predictions.

\smallskip
\noindent\textbf{Related Work }
Our work is situated in the broad area of adversarial machine learning, which is concerned with adversarial tampering with data used for either learning or predictions~\citep{huang2011adversarial,VorobeychikBook18}.
The most significant progress in this field has been in the context of supervised learning, with a broad array of approaches proposed to improve robustness of learning to adversarial perturbations to either training data~\citep{kearns1993learning,long2011learning,liu2017robust} or to inputs at prediction time~\citep{Cai18,Madry18}.
While many of the approaches for improving ML robustness focus on empirical robustness measures, a  number aspire to formally verify robustness to adversarial perturbations~\citep{chiang2020detection,Cohen2019CertifiedAR,wang2021betacrown,xu2021fast}, and to train models that exhibit improved \emph{verified} or \emph{certified}, rather than merely \emph{empirical} robustness~\citep{salman2019provably,Wong18}.
Our contribution is most connected to this latter line of work.

Modern reinforcement learning approaches have also been shown to suffer from adversarial perturbations~\citep{pattanaik2017robust,sun2020stealthy}.
A number of approaches have been proposed to improve RL robustness~\citep{lutjens2020certified,oikarinen2021robust,wu2022robust,zhang2020robust}, although the efficacy of these varies a great deal by domain.
Several recent efforts address the robustness certification problem in reinforcement learning by adapting the randomized smoothing technique from supervised to reinforcement learning~\citep{kumar2022policy,wu2022crop}.
However, these are limited to certifying either policy invariance or reward lower bounds, rather than 
safety.

There is a rich literature in control theory and formal methods tackling the settings of adversarial perturbations of states and observations of dynamical systems. Research problems such as stability, reachability, robustness, formal safety verification, and correct-by-construction control synthesis in such settings have been extensively studied in the past few decades 
\citep{MitraCPSBook2021,tabuada_2009,C2E2paper,FanQMV:CAV2017,CAS13,Spaceex,bak2017hylaa,Althoff2015a,TIRA_Meyer_2019}. 
More recently, a promising line of research resulted in methods for verifying safety properties input-output inclusion properties of various neural networks (NN) and for dynamical systems with machine-learned, usually NN, controllers \citep{ivanov2019verisig,Sherlock-poly-2019,NNV,reachnn}. 
However, most prior work has assumed relatively low-dimensional systems and inputs.
There is an emerging line of research focused on assuring the robustness and safety of control systems 
with high-dimensional sensor inputs
\citep{Dean2020GuaranteeingSO,pmlr-v120-dean20a,pmlr-v144-dean21a,Ulices-Shoukry-NNLander-VeriF-2022,Katz-Corso-Strong-Kochenderfer-2022}. So far, the results are restricted to scenarios where
the map from sensor inputs to semantic outputs of the perception module is smooth, %
the state remains near training data of the perception module, 
an accurate generative model mapping semantic states to realistic high-dimensional sensor inputs is available, or 
an accurate model of the environment is known.
\section{Model}
\label{S:model}

We build on the formal model of deterministic finite-horizon \emph{partially observable Markov decision processes (POMDPs)}. 
A POMDP is defined by a tuple $(S,A,O,r,\tau,f,d)$, where $S$ is the state space, $A$ is the set of actions, $O$ is the set of observations, $r(s,a)$ is the reward function, which is a function of the current state $s$ and action $a$, $\tau(s,a)$ is the (deterministic) transition function that returns the next state $s'=\tau(s,a)$,
$f(s)$ is the deterministic observation function, 
with $o=f(s)$ for observation $o$ and state $s$, and $d$ is the probability distribution over the initial states.
Moreover, we consider the setting where the agent must also abide by a \emph{safety constraint} that certain unsafe states should never be visited.
Let $U$ denote the set of all unsafe states $s$.

We assume that $d$, $P$, $f$, and $r$ are unknown, and the agent must therefore learn how to act in this POMDP environment.
In addition, we assume that the state space $S \subseteq \mathbb{R}^m$ and the observation space $O \subseteq \mathbb{R}^n$, where $n \gg m$.
This setup captures settings in which the true state $s$ is comprised of a collection of semantically meaningful variables, while
the observations $o=f(s)$ are based on high-dimensional perceptual inputs, such as images, LiDAR point clouds, etc.

At any point in time $t$, the agent knows the history of observations, $h_t = \{o_0,\ldots,o_t\}$. 
In general, a policy $\pi$ depends on the full history of observations.
If the state and observation spaces are finite, one can use \emph{belief state} as a sufficient statistic. However, this is difficult in our setting since 
observations are high-dimensional.
Consequently, a common approach in applying deep reinforcement learning to partially-observable environments with high-dimensional observations is to condition policies only on a finite sequence of preceding observations.
To simplify exposition, we assume that policies depend only on the latest observation $o$; extension of our approaches to dependence on finite histories is straightforward.
We denote a policy by $\pi(o)$, mapping an observation $o$ to an action $a$.
By treating $o$ effectively as a proxy for the true state $s$, we can now apply any standard reinforcement learning method directly; and, indeed, this has been done with considerable success~\citep{Bojarski16,peng2021end,zhang2019end}.

Define $H_T(\pi) = \{o_0,s_0,o_1,s_1,\ldots,o_T,s_T\}$ to be the sequence of observations and states induced by a policy $\pi$, and suppose that $S_T(\pi) \subset H_T$ be the set of states in $H_T(\pi)$.
The goal of the agent is to maximize the total expected sum of rewards over a finite time horizon $T$, subject to the constraint that no unsafe states are visited along the way.
Formally, the agent's goal is to solve
\begin{align}
\label{E:goal}
\max_\pi \ \mathbb{E}\left[\sum_{t=0}^T r_t|\pi\right] \quad \mathrm{s.t.}: \quad  U \cap S_T(\pi) = \emptyset,
\end{align}
where the expectation is with respect to any randomness in the agent's decisions and initial states.

Our primary focus is robustness to adversarial perturbations to observations $o$.
Specifically, let $\delta$ be the adversarial perturbation which results in the observed state $o' = o + \delta$.
As is common~\citep{VorobeychikBook18}, we assume that the magnitude of $\delta$ is constrained to be in an $\ell_p$ $\epsilon$-ball, i.e., $\|\delta\|_p \le \epsilon$, for an exogenously specified $\epsilon \ge 0$ and norm $p \in \mathbb{N} \cup \{\infty\}$.
The adversary's goal is to cause the violation of the safety constraint over the finite POMDP horizon.
To formalize, let $\Delta=\{\delta_0,\ldots,\delta_{T-1}\}$ be the sequence of adversarial perturbations, $H_T(\pi,\Delta)$ be the history of the perturbed sequence of observations and states, and $S_T(\pi,\Delta)$ be the associated subset of states.
The adversary's goal is to identify a perturbation sequence $\Delta$ such that $\|\delta_t\|_p \le \epsilon$ for all $t$, and $U \cap S_T(\pi,\Delta) \ne \emptyset$.

\section{Partially-Supervised Reinforcement Learning}

In modeling the underlying environment as a deterministic POMDP, we have introduced an explicit way to refer 
to (true/semantic) states $s$ over which safety constraints $U$ 
are defined.
One can still, in principle, make use of conventional deep RL techniques to learn policies $\pi$ that map observed inputs to low-level control $a \in A$.
However, there are a host of reasons why this is undesirable in safety-critical domains.
First, when things go wrong, it is difficult to use the policy $\pi(o)$ to yield actionable insights into why, as it does not provide sufficient semantic information.
Second, there are often large datasets from complementary domains that allow us to learn semantically meaningful mappings $g(o)$ from sensory inputs, such as images, to states, such as the location of the ego vehicle, as well as other objects and vehicles in the scene.
Consequently, it is common in such domains to eschew end-to-end RL in favor of approaches that compose perceptual learning and reasoning with learning how to act.

We propose a \emph{partially-supervised reinforcement learning (PSRL)} framework as a stylized way to capture the typically far more complex compositional approaches for learning to act in 
high-stakes domains.
The key idea in PSRL is to take advantage of many situations in which the true state $s$ \emph{is known at training time}, and only unknown at decision time.
For example, if we are to train an autonomous vehicle using simulations, the true state is actually available. It is only as we transition a working vehicle to a physical environment that it must navigate with only perceptual information.
In PSRL, we construct a policy $\pi(o)$ as a composition of perceptual prediction module $g(o)$ and a semantic policy $\pi_s(s)$ which maps the predicted state $s$ to an action.
Thus, the goal in PSRL is to learn both $g$ and $\pi$. At decision time, we implement $\pi(o) = \pi_s(g(o))$.
Significantly, we can apply supervised learning to learn $g(o;\theta_g)$, which maps high-dimensional observations $o$ to relatively low-dimensional states $s$, and can apply standard RL to learn $\pi_s(s;\theta_s)$ over the low-dimensional predicted states $s$.
Below, we make use of standard regression loss for $g$, with $l_2$ loss $l(g(o),s) = \|g(o)-s\|_2^2$, and use DDQN~\citep{VanHasselt2016} deep Q-learning approach to learn an action-value function $Q(s,a)$, with $\pi_s(s) = \arg\max_a Q(s,a)$.
However, any standard approach for these would work.
During 
each training iteration, we update both $g$ using the supervised loss and $Q$ using the standard DDQN approach.

In the sequel, we address two key problems in using PSRL to obtain high-quality policies that remain safe even under adversarial perturbations to high-dimensional inputs $o$:
\begin{enumerate}
\item Formally certify safety properties of learned policies at decision time with respect to adversarial perturbations. %
\item Learn policies with strong certified safety guarantees under adversarial perturbations.
\end{enumerate}

\section{Certifying Safety under Adversarial Perturbations}

We begin with the problem of certifying safety of PSRL policies, $\pi(o) = \pi_s(g(o))$.
The goal of robustness certification in supervised learning is to either ascertain whether a prediction is robust (i.e., invariant) under a given perturbation budget $\epsilon$, or to certify a budget $\epsilon$ so that no perturbation under this budget constraint can possibly cause the prediction to change.
Two extensions of this idea to reinforcement learning involve either certifying that a policy is invariant under adversarial perturbations, or certifying a lower bound on total reward.
In our context, however, our sole concern is to certify \emph{safety} under adversarial perturbations.
The reason is that we view adversarial encounters as exceedingly rare, and expect the nominal policy to work well in most cases when it is deployed (as otherwise, it would not even be deployed).
However, when these do occur, we are typically far less concerned with issues such as efficiency (which is what rewards most naturally capture) and far more concerned with issues of safety (that the vehicle, for example, doesn't crash).

We formally define the \emph{adversarial safety certification (ASC)} problem as follows: 
\begin{center}
\emph{Given a policy $\pi$, identify the highest $\epsilon^*(\pi)$ such that} $U \cap S_T(\pi,\Delta)=\emptyset$ for all $\Delta$ such that $\|\delta_t\|_p \le \epsilon^*(\pi)$ and for all $0 \le t \le T-1$.
\end{center}
In practice, it will typically be infeasible to actually find $\epsilon^*(\pi)$, and our goal will be to return a provable lower bound $\bar{\epsilon} \le \epsilon^*(\pi)$.
If even the nominal (that is, unperturbed) policy is unsafe, $\epsilon^*(\pi) \equiv 0$.
Clearly, the \emph{ASC} problem definition requires the knowledge of whether a state $s$ encountered is unsafe.
We assume that this information is easy to obtain when one actually enters an unsafe state (e.g., when the car crashes, it is not difficult to verify that it has crashed).\footnote{In the case where we only simulate potential future during decision-time certification, we can alternatively leverage supervised learning to predict whether we have entered an unsafe state.  There is no concern about adversarial perturbations here, since this is done only for the purposes of certification and does not involve actual interaction with the environment.}

We assume that we 
can simulate transition $\tau(s,a)$ for any state $s$ and action $a$.
This is a significantly weaker assumption than those commonly made in verification of dynamical systems literature, where either the dynamical model is assumed to be fully known, or one assumes specific structure (such as discrepancy functions)~\citep{FanQMV:CAV2017,hsieh2021verifying,Ulices-Shoukry-NNLander-VeriF-2022,sun2019formal}.
Next, we present our algorithm for adversarial safety certification, \textsc{TASC} (tree-based adversarial safety certification algorithm).

Consider a given observation $o$, policy $\pi_s(s)$, state prediction function $g(o)$, and a bound on adversarial perturbation $\epsilon$.
The key building block in the certification procedure is to return a set of actions $A_\epsilon(o)$ with the property that the adversary cannot cause the policy to select an action not in that set $A \setminus A_\epsilon(o)$.
In the PSRL framework, this entails two steps: first, we certify a set of states $S_\epsilon(o)$ such that $g(o) \in S_\epsilon(o)$, and second, we certify a set of actions $A_\epsilon(o)=A(S_\epsilon(o))$ as the only actions that any state $s \in S_\epsilon(o)$ can induce.
Our approach is to treat these two steps as separate robustness certification problems under $\ell_p$-bounded input perturbations.
Specifically, our certification of the policy $\pi_s$ takes the smallest $\ell_\infty$-norm $\gamma$ such $S_\epsilon(o) \subseteq B_\gamma(s)$, where $B_\gamma(s)$ is a $\gamma$-ball around the predicted state $s = g(o)$.
We can then use, e.g., the $\alpha,\beta$-CROWN approach \citep{xu2021fast,wang2021betacrown} to calculate the upper/lower bound of Q values, $\overline{Q}(s,a)$ and $\underline{Q}(s,a)$ for each action $a$ for $s=g(o)$, on which $\pi_s$ is based (recall that we assume the use of Q-learning).
We can then use the bounds on the Q values to determine the action set $A_\epsilon(o)$ as the set of all actions $a$ for which $\overline{Q}(s,a) \ge \max_{a'} \underline{Q}(s,a')$.

For the problem of certifying $g(o)$ above, we appeal to two distinct approaches depending on the norm $\ell_p$.
Specifically, if we consider adversarial perturbations to observations $o$ to be bounded in $\ell_\infty$ norm, we can apply, e.g., the $\alpha,\beta$-CROWN approach.
On the other hand, if perturbations are bounded in $\ell_2$ norm, we can make use of median smoothing~\citep{chiang2020detection} to obtain interval bounds for each coordinate of $s$, which is now the median of the smoothed prediction.
Either approach in fact directly yields an $\ell_\infty$-ball $B_\gamma(s)$ that we can then use to certify the policy $\pi_s$ as described above.

Armed with this setup for certifying an action set for a given observation and $\epsilon$, we can now describe our full ASC algorithm.
ASC proceeds through a pre-defined discrete set $\{0,\epsilon_1,\epsilon_2,\ldots,\epsilon_N\}$ (e.g., by discretizing $\epsilon$ between 0 and 1 if inputs are normalized to the [0,1] interval).
Let $T_v$ be the verification horizon (which may be different from the reward horizon $T$), and consider the nominal history $H_{T_v}(\pi)$ and associated set of states $S_{T_v}(\pi)$ generated by the policy $\pi$ we are verifying from a known initial state $s_0$.\footnote{The assumption that the initial state is known is without loss of generality, since we can always define a dummy initial state before the beginning of the scenario.}
The first step is to check whether $U \cap S_{T_v}(\pi) = \emptyset$; if not, we simply return $\epsilon^*(\pi)=0$.
Otherwise, we proceed through a series of iterations of certification.

In each iteration $i$, we know to have verified $\bar{\epsilon} = \epsilon_{i-1}$, and we attempt to verify $\epsilon_{i}$.
In addition, through iterations $0,\ldots,i-1$, we have expanded a tree defined by sequences of actions beginning from the initial state $s_0$ (root node of the tree).
This tree is typically a very sparse subtree of the full tree defined by all possible $T_v$-long sequences of actions starting with $s_0$.
Each node $r$ in this tree corresponds to a state $s_r$ (equivalently, a sequence of actions taken starting with $s_0$), as well as the associated observation $o_r$.
In addition, for each node $r$ we have previously verified a set of actions $A_r^{i-1}$.
Since we know the initial state, we begin with the node that follows, and proceed along the nominal path $H_T(\pi)$.
For each node $r$ thereby encountered, we certify the set of actions $A_r^{i}$ under $\epsilon_i$.
If $A_r^{i} = A_r^{i-1}$, we can update the certificate for node $r$ to $\epsilon_i$ (and do not need any additional exploration of the subtree rooted at $r$).
If $A_r^{i} \ne A_r^{i-1}$, let $A_{r,i}^i$ be the set of new actions added by $\epsilon_i$ as compared to $\epsilon_{i-1}$.
Each action $a \in A_{r,i}^i$ induces a nominal history $H_{T(r)}(\pi,r)$, where $T(r)$ is the remaining verification horizon starting at node $r$, with associated set of states $S_{T(r)}(\pi,r)$.
If $U \cap S_{T(r)}(\pi,r) \ne \emptyset$, we immediately return $\bar{\epsilon} = \epsilon_{i-1}$, since we can no longer verify safety for $\epsilon_i$.
Otherwise, we expand the subtree rooted at $r$ one node at a time in a depth-first fashion, certifying an action set for each node under $\epsilon_i$, and return as soon as any unsafe state is encountered.
If the entire subtree is expanded without observing an unsafe state, we verify $\epsilon_i$ for node $r$, and continue along the nominal trajectory.
If we successfully complete $T_v$ steps along the nominal trajectory in this fashion, we can now verify $\bar{\epsilon} = \epsilon_i$, and we continue with the next iteration.
The complete \textsc{TASC} algorithm is provided in the full version of the paper (see \url{https://vorobeychik.com/psrl_safety_cert.pdf}).

\textsc{TASC} has three important properties.
First, it is straightforward to make it ``real-time'', that is, to have it return a certified adversarial safety bound $\bar{\epsilon}$ after an arbitrary time limit.
Second, as long as the procedure for certifying a set of actions $A_\epsilon(o)$ is sound, \textsc{TASC} is also sound.
Third, if $A_\epsilon(o)$ is also complete, then \textsc{TASC} is complete.
Next, we formalize the latter two properties.

Let \textsc{Cert}($\pi,o,\epsilon$) return a set of action $A_\epsilon(o)$ that are certified robust for policy $\pi$ in observation $o$ under a given bound on adversary's budget $\epsilon$ for some given norm $\ell_p$.
\begin{definition}
\textsc{Cert}($\pi,o,\epsilon$) is \emph{sound} if there does not exist an action $a' \notin A_\epsilon(o)$ 
where $\exists \delta$ with $\|\delta\|_p \le \epsilon$ such that $\pi(o+\delta)=a'$.  It is \emph{complete} if 
$\forall a' \in A, \exists \delta : \pi(o+\delta) = a'$ and $\|\delta\|_p \le \epsilon \Rightarrow a' \in A_\epsilon(o)$.
An ASC algorithm is \emph{sound} if it returns $\bar{\epsilon} \le \epsilon^*(\pi)$.
It is $\gamma$-\emph{complete} if $\epsilon^*(\pi) - \bar{\epsilon} \le \gamma$. 
\end{definition}
\begin{proposition}
If \textsc{Cert}($\pi,o,\epsilon$) is sound, then \textsc{TASC} is sound.  If \textsc{Cert}($\pi,o,\epsilon$) is also complete, then \textsc{TASC} is $\gamma$-complete, where $\gamma = \max_{i\ge 1} (\epsilon_i - \epsilon_{i-1})$ for the finite discrete schedule of adversarial noise bounds $\{\epsilon_i\}$, where $\epsilon_N$ is the absolute upper bound.
\end{proposition}
\begin{proof}[Proof Sketch]
Suppose \textsc{Cert}($\pi,o,\epsilon$) is sound, and let \textsc{TASC} returns $\bar{\epsilon}=\epsilon_i$ for some $i$ in the discrete sequence of adversarial noise bounds.
Consider the tree expanded before the algorithm moves forward with $\epsilon_{i+1}$. Since \textsc{Cert}($\pi,o,\epsilon$) is sound, the tree has expanded all actions in $A_\epsilon(o)$ for every node $o$ through depth $T_v$, and there can be no more actions to expand.
Consequently, the tree contains all nodes that the adversary could reach with $\bar{\epsilon}$.
Since \textsc{TASC} did not identify any of these to be unsafe, soundness follows.

For completeness, consider the same setting as above, and note that either $i=N$, in which case it's complete since $\epsilon_N$ is the absolute upper bound, or $i < N$, in which case a subtree generated for $\epsilon_{i+1}$ was found to contain an unsafe state.
Since \textsc{Cert}($\pi,o,\epsilon$) is complete, it follows that no actions were expanded unnecessarily. Thus, there exists a path from the starting state to an unsafe state that an adversary can induce with perturbations up to $\epsilon_{i+1}$.
This means that $\epsilon^*(\pi) \in [\epsilon_i,\epsilon_{i+1})$, and the result follows.
\end{proof}

\section{Learning Robust Policies}
\label{S:at}

A natural complement to certification is the problem of bolstering robustness through learning.
As mentioned above, the goal of conventional learning is to maximize total reward subject to the safety constraints (Equation~\eqref{E:goal}).
This learning problem, however, is focused solely on nominal reward, and only avoids safety violations on the nominal (that is, prior to adversarial perturbations) sequence of states.
In certified robust learning, we add a third goal: to maximize the certification bound $\epsilon^*(\pi)$ of the learned policy $\pi$.
Thus, the combined problem is
\begin{align}
\label{E:robustgoal}
\max_{\pi : U \cap S_T(\pi, \Delta) = \emptyset} \ \alpha \mathbb{E}\left[\sum_{t=0}^T r_t|\pi\right] + (1-\alpha) \epsilon^*(\pi),
\end{align}
where $\Delta = \{\delta_t\}$ with $\|\delta_t\|_p \le \epsilon^*$.

Note that Problem~\eqref{E:robustgoal} is substantively different from the typical goal of robust reinforcement learning, which is to maximize worst-case sum of rewards, where certification would simply play the role of constructing a lower bound on the worst-case cumulative reward.
We maintain, rather, that it is crucial to separate the primary goal operative \emph{under normal circumstances}, which is to maximize the sum of nominal rewards, and robustness pertains primarily to critical considerations, such as violation of safety constraints, in the \emph{rare} adversarial encounters.

One of the most basic techniques for boosting robustness of machine learning techniques---including reinforcement learning---is adversarial training~\citep{Madry18,oikarinen2021robust}.
In DDQN, adversarial training augments the DDQN loss with an adversarial counterpart.
There are a number of ways to implement the adversarial loss, but the most basic is to replace the actor $Q$ function, $Q_{\mathrm{actor}}(o,a)$ (using observation $o$ as a proxy for state, as is common) with its adversarial counterpart $Q_{\mathrm{actor}}(o+\delta,a)$, where the adversarial perturbation $\delta$ is generated by solving $\min_\delta \sum_a \pi_a(o)Q_{\mathrm{actor}}(o+\delta,a)$ using projected gradient descent, where $\pi_a(o)$ is the policy induced by the actor $Q$ function by passing the $Q$ function through a softmax layer~\citep{wu2022robust}.
We refer to this baseline version of adversarial training as simply \textsc{AT}.

An alternative, proposed by \cite{oikarinen2021robust}, is to use provably bounds on the actor $Q$-function obtained using interval bound propagation (IBP)~\citep{gowal2018effectiveness}
to maximize the difference the lower bound on the $Q$ function value for the action prescribed by the nominal policy and the upper bound $Q$ value for every other action.
We refer to this approach as \textsc{Radial}.
Note that both \textsc{AT} and \textsc{Radial} effectively collapse the PSRL architecture into a single $Q$ function (and, consequently, policy) architecture.

In addition, we propose two novel approaches to adversarial training that take advantage of the PSRL compositional architecture.
The first, \textsc{PSRL-AT}, makes use of IBP in both the loss for the policy $\pi_s$ (or, rather, the associated $Q$ function), as well as the supervised loss for $g$.
For the former, the approach is analogous to \textsc{Radial}.
For the latter, the adversarial counterpart of the $\ell_2$ loss becomes $\|\underline{g}(o)-s\|_2^2+\|\overline{g}(o)-s\|_2^2$, where $\underline{g}$ and $\overline{g}$ are the upper and lower bounds on $g$, respectively, which can be computed by IBP in the case of $\ell_\infty$ loss.
In the case of $\ell_2$ loss, we take an approach similar to \cite{salman2019provably}, combining randomized smoothing with the approach for generating bounds on the actor $Q$ function described above.

Our second proposed approach is a hybrid of \textsc{Radial} and \textsc{PSRL-AT}.
Specifically, we first use the \textsc{Radial} loss function to update the full composite $Q$ function (backpropagating all the way through the supervised component $g$).
In addition, we use \emph{only the adversarial supervised loss} $g$ to perform a second update over the same minibatch.
We refer to this approach as $\textsc{PSRL-Hybrid}$.
\section{Experiments}

\noindent\textbf{Experiment Setup }
Our experiments use the \emph{Highway} autonomous driving environment~\citep{highway-env}, which we customized to our setting (see the full version of the paper for details).
We experiment on three scenarios within this environment: \emph{Highway}, \emph{Twoway}, and \emph{Exit}.
In our experiments, we assume that the position of the ego vehicle is known to slightly simplify the setup.
In \emph{Highway} scenario, all vehicles are heading in the same direction on a three-lane road, with vehicles other than the ego vehicle traveling at constant speed.
The ego vehicle receives the reward of 0.1 for staying in the rightmost lane, 0.5 for high speed (between 20-30), and -1 for reaching an unsafe state, which we define as either collision or being within 1 unit of distance of other vehicles.
In the \emph{Twoway} scenario, the vehicles drive on a two-lane road in opposite directions.
There are 3 vehicles at the same lane as the ego vehicle, and 2 in the opposite lane.
The ego vehicle receives the reward of 0.5 for high speed, with the same penalty for unsafe states and collision as above. 
The \emph{Exit} scenario is similar to the \emph{Highway} scenario, except there is an exit lane appearing on the right, with a reward of 1 for a successful exit.
We assume we know the distance to the exit.

Our safety certification uses $T_v \in \{5,10,15\}$.
PSRL in our context involves two neural networks: $g(o)$, which maps the image input into the positions of other vehicles in the observed scene, and the $Q$ network, which takes the predicted state $s$ (i.e., predicted positions of two vehicles closest to ego vehicle in the scene) as input. However, we always use the true state to check for safety violation and collision.
First, we use standard DDQN~\citep{VanHasselt2016} to train the $Q$ network.
We train for 2000 episodes in the \emph{Highway} scenario, for 3000 episodes for the \emph{Twoway} scenario, and 50,000 steps for the \emph{Exit} scenario.
Next, we use the policy learned by DDQN to sequentially train $g$, using a $\gamma$-greedy policy (i.e., following the learned policy $\pi_s(s)$ with probability $1-\gamma$, and taking a random action otherwise), with parameter $\gamma$ decreasing over time as in DDQN.
We train $g$ for 500 episodes in the \emph{Highway} scenario, for 1500 episodes in the \emph{Twoway} scenario, and for 1000 episodes in the \emph{Exit} scenario.

For all our variants of adversarial training, we  train for 40,000 steps. The $\epsilon$ for adversarial training is increased linearly over the first 25,000 steps and then remains at the target level thereafter. 
We consider three values for a target $\epsilon$ for adversarial training: $\{1,2,3\}/255$.
Here we present the results for $\ell_\infty$ perturbations; results for $\ell_2$, and 
full details of the experiment setup, are provided in the full version of the paper.

\smallskip
\noindent\textbf{Efficacy of TASC} We begin by evaluating the efficacy of the proposed \textsc{TASC} adversarial safety certification algorithm. We set the maximum total number of nodes allowed to be explored to $M=500$; however, this was saturated in only a few cases in our experiments, and does not affect the medians reported below. 
We find that even for a relatively short horizon, \textsc{TASC} only needs to explore a small fraction of the search tree, in all cases below 5\%, and in many cases under 1\% of the nodes.
This becomes even more pronounced as the horizon increases.
For $T_v = 10$, \textsc{TASC} explores fewer than 0.0001 fraction of all nodes in all cases in all scenarios, while when $T_v = 15$, it explores fewer than 0.00001 fraction of nodes. 
Complete details are provided in the full version of the paper.

\smallskip
\noindent\textbf{Efficacy of Adversarial Training}
Next, we evaluate the efficacy of the adversarial training approaches described in Section~\ref{S:at}.
We compare \emph{PSRL} with no adversarial training, which we refer to as \emph{Vanilla}, \emph{AT} (standard adversarial training), \emph{RADIAL}~\citep{oikarinen2021robust} adapted to our setting by backpropagating adversarial error through the combined $\pi(o) = \pi_s(g(o))$ policy, \emph{PSRL-AT}, and \emph{PSRL-Hybrid}.
We consider three measures of efficacy: nominal reward, certified safety, 
and three measures of mean squared error (MSE) of state prediction: prediction based on actual observation $o$ (\emph{avg err}), and upper (\emph{avg err (ub)}) and lower (\emph{avg err (lb)}) bounds of prediction error after adversarial perturbation with $\epsilon = 1/255$, computed by $\beta$-CROWN~\citep{wang2021betacrown}.

Figure~\ref{F:reward-cert} presents the results for nominal reward (left column) and certified safety (right column).
\emph{PSRL} has a reliably high nominal reward, but neither \emph{PSRL} nor \emph{AT} yield meaningful safety certification, as we would expect, with AT often also suffering in terms of nominal reward.
\emph{PSRL-Hybrid} consistently performs well on both nominal reward and safety certification, while both \emph{RADIAL} and \emph{PSRL-AT} vary in this regard, scoring the best in some cases, and relatively poorly in others.
Thus, the hybrid approach indeed appears to yield the best balance of nominal reward and certified safety across the three scenarios.

\begin{figure}[h!]
\centering

\begin{tabular}{ccc}
\includegraphics[width=0.55\linewidth]{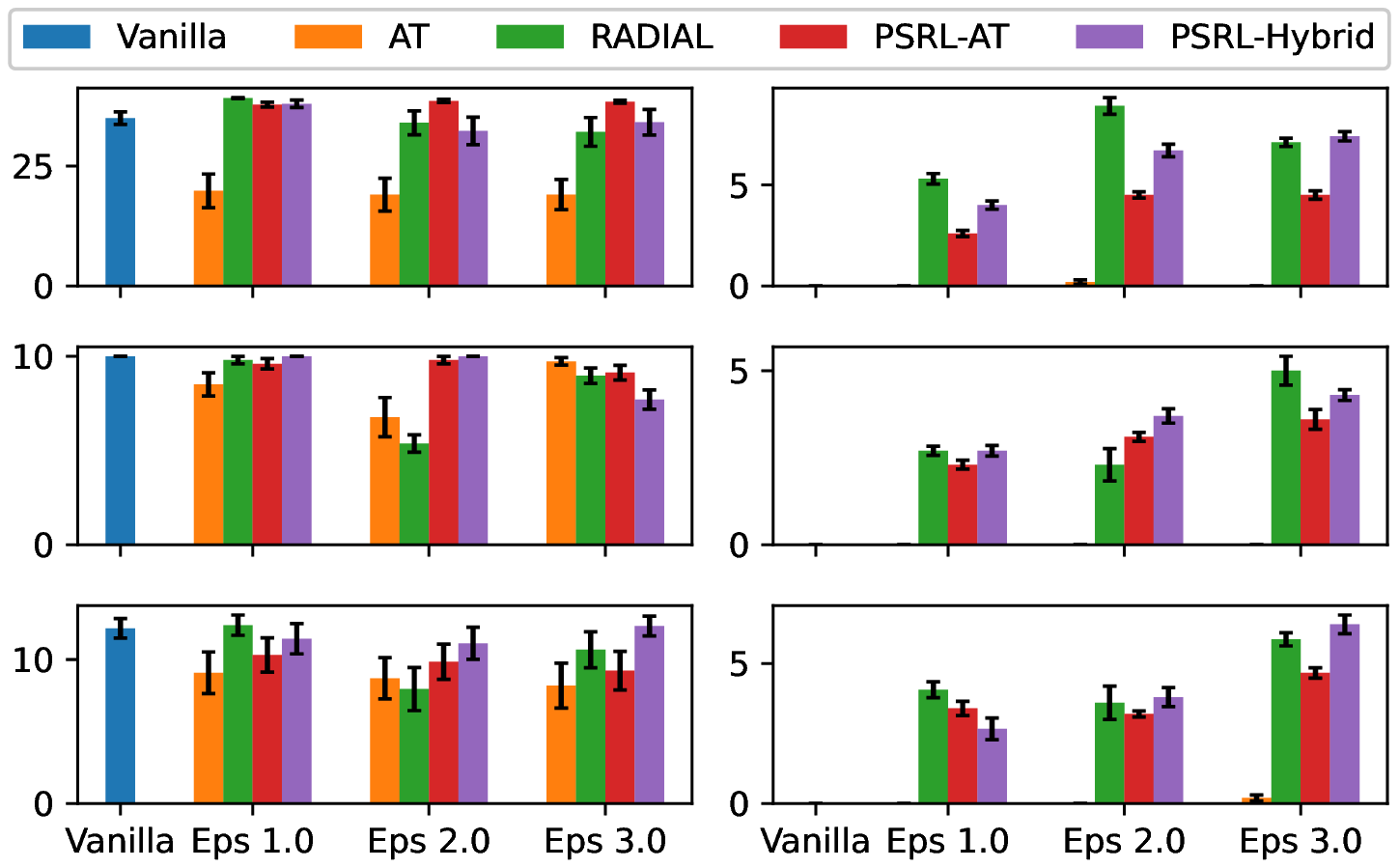}
\includegraphics[width=0.3\linewidth]{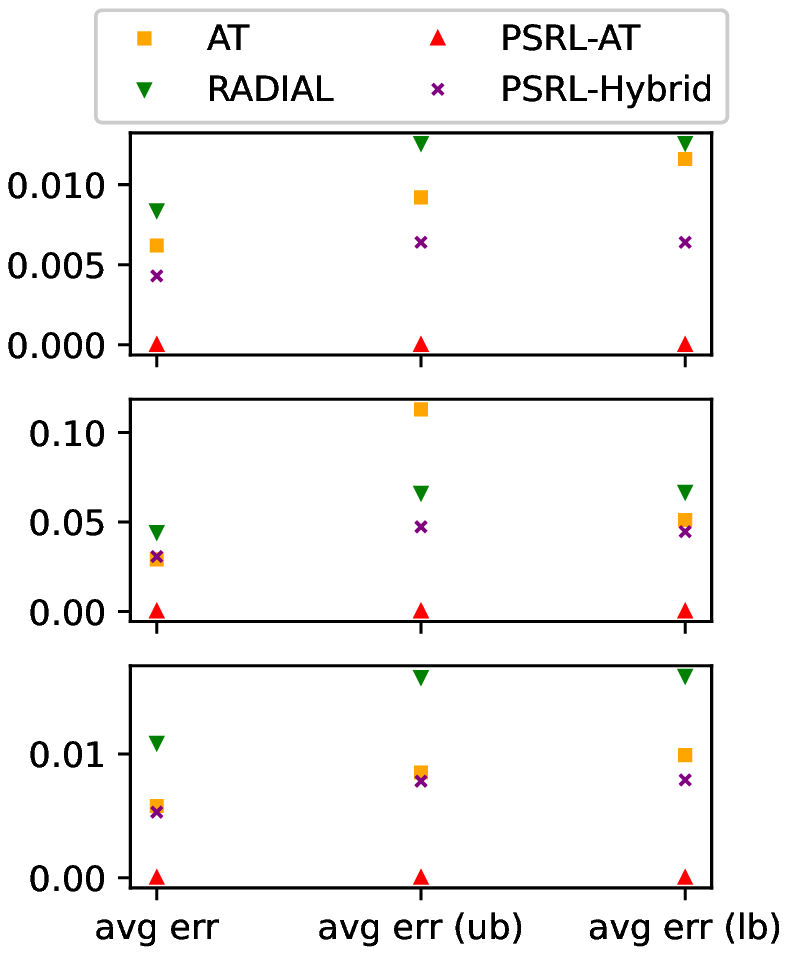}
\end{tabular}
\caption{Average reward (left), certified safety (middle), and MSE of state prediction error (right) for Highway (top), Twoway (middle) and Exit (bottom).
}
\label{F:reward-cert}
\end{figure}

When it comes to MSE, both \emph{PSRL-AT} and \emph{PSRL-Hybrid} typically outperform other adversarial training alternatives, and \emph{PSRL-AT} in particular is far better than all alternatives.
\emph{PSRL}-style adversarial training methods thus yield the additional advantage of consistency in performance and interpretability (in correctly predicting state based on observations) over end-to-end approaches.

\vspace{-10pt}
\section{Conclusion}

We presented the first framework for adversarial safety certification in reinforcement learning in the context of a POMDP system model.
We present a highly effective adversarial safety certification algorithm, TASC, as well as several adversarial training methods, several of which rely on a novel composition of supervised and reinforcement learning, PSRL, which leverages observations of true state at training time.
We show that making use of PSRL in adversarial training can yield strong nominal performance, high adversarially certified safety, and also obtain high-quality predictions of state, enabling a high level of interpretability.

\bibliography{aaai23,hussein}

\end{document}